\documentclass[twoside]{article}
\usepackage[accepted]{aistats2016}

\usepackage{hyperref}
\usepackage{url}
\usepackage{amsmath,amsfonts,amsthm,amssymb}
\usepackage{multirow,mdwlist}
\usepackage{graphicx,epsfig,subfigure,epsfig}
\usepackage{color}
\usepackage{algorithm,algorithmic}
\newtheorem{theorem}{Theorem}
\newtheorem{lemma}{Lemma}
\newtheorem{hyp}{Hypothesis}
\newtheorem{defin}{Definition}
\newtheorem{pro}{Proposition}

\newcommand\mycorr[1]{{#1}}

\newcommand\projecturl{\url{https://github.com/yjq8812/aistats2016}}

\begin{document}

\twocolumn[

\aistatstitle{A Convex Surrogate Operator for General Non-Modular Loss Functions}

\aistatsauthor{
  Jiaqian Yu
\And
Matthew B.\ Blaschko
}

\aistatsaddress{
Inria \& CentraleSup\'{e}lec, Universit\'{e} Paris-Saclay\\
Grande Voie des Vignes\\
92295 Ch\^{a}tenay-Malabry, France\\
\texttt{jiaqian.yu@centralesupelec.fr}
\And
Center for Processing Speech and Images\\
Departement Elektrotechniek, KU Leuven\\
3001 Leuven, Belgium\\
\texttt{matthew.blaschko@esat.kuleuven.be}
} ]

\begin{abstract}
Empirical risk minimization frequently employs convex surrogates to underlying discrete loss functions in order to achieve computational tractability during optimization. However, classical convex surrogates can only tightly bound modular loss functions, submodular functions or supermodular functions separately while maintaining polynomial time computation. In this work, a novel generic convex surrogate for general non-modular loss functions is introduced, which provides for the first time a tractable solution for loss functions that are neither supermodular nor submodular. This convex surrogate is based on a submodular-supermodular decomposition for which the existence and uniqueness is proven in this paper. It takes the sum of two convex surrogates that separately bound the supermodular component and the submodular component using slack-rescaling and the Lov\'{a}sz hinge, respectively. It is further proven that this surrogate is convex, piecewise linear, an extension of the loss function, and for which subgradient computation is polynomial time. Empirical results are reported on a non-submodular loss based on the S{\o}rensen-Dice difference function, and a real-world face track dataset with tens of thousands of frames, demonstrating the improved performance, efficiency, and scalability of the novel convex surrogate. 
\end{abstract}

\section{Introduction}

Many learning problems involve the simultaneous prediction of multiple labels.  A simple strategy is to empirically minimize the Hamming loss over the set of predictions~\cite{Taskar2003}.  However, this does not always reflect the underlying risk of the prediction process, and may lead to suboptimal performance.  Following the risk minimization principle~\cite{Vapnik:1995}, we may instead wish to minimize a loss function that more closely reflects the cost of a specific set of predictions.  Alternatives to the Hamming loss are frequently employed in the discriminative learning literature: \cite{cheng2010bayes} uses a rank loss which is supermodular; \cite{petterson2011submodular} uses a non-submodular loss based on F-score; \cite{doppa2014hc} uses modular losses e.g.\ Hamming loss and F1 loss  which is non-submodular; and losses that are nonmodular are common in a wide range of problems, including Jaccard index based losses~\cite{BlaschkoECCV2008,EveringhamIJCV2010,NowozinCVPR2014}, or more general submodular-supermodular objectives~\cite{narasimhan2005subsup}.

This has motivated us to study the conditions for a loss function to be tractably upper bounded with a tight convex surrogate.  For this, we make use of the discrete optimization literature, and in particular submodular analysis~\cite{fujishige2005submodular,Schrijver2003}.
Existing polynomial-time convex surrogates exist for supermodular~\cite{tsochantaridis2005large} or submodular losses~\cite{yuICML2015}, but not for more general non-modular losses.
We may perform approximate inference in polynomial time via a greedy optimization procedure to compute a subgradient or cutting plane of a convex surrogate for a general increasing function, but this leads to poor performance of the training procedure in practice~\cite{Finley:2008,Joachims/etal/09a}. A decomposition-based method for a general set function has been proposed in the literature~\cite{iyer2012algorithms}, showing that under certain conditions a decomposition into a submodular plus a supermodular function can be efficiently found. Other relevant work includes the hardness results on submodular Hamming optimization and its approximation algorithms~\cite{gillenwater2015nips}.

In this paper, we propose a novel convex surrogate for general non-modular loss functions, which is solvable for the first time for non-supermodular and non-submodular loss functions. 
In Section~\ref{sec:method}, we introduce the basic concepts used in this paper. In Section~\ref{sec:convexsurrogate}, we define a decomposition for a general non-modular loss function into supermodular and submodular components (Section~\ref{sec:decomposition}), propose a novel convex surrogate operator based on this decomposition (Section~\ref{sec:ConvexSurrogateDefinition}), and demonstrate that it is convex, piecewise linear, an extension of the loss function, and for which subgradient computation is polynomial time (Section~\ref{sec:ConvexSurrogateProperties}). In Section~\ref{sec:DiceTheory}, we introduce the S{\o}rensen-Dice loss, which is neither submodular nor supermodular. In Section~\ref{sec:result} we demonstrate the feasibility, efficiency and scalability of our convex surrogate with the S{\o}rensen-Dice loss on a synthetic problem, and a range of non-modular losses on a real-world face-track dataset comprising tens of thousands of video frames.

\section{Non-modular loss functions}\label{sec:method}

In empirical risk minimization for a set of binary predictions, 
we wish to minimize some functional of
\begin{equation}
\hat{\mathcal{R}}(h) = \frac{1}{n} \sum_{i=1}^n \Delta(y_i,\operatorname{sign}(h(x_i))) .
\end{equation}
For an arbitrary loss function 
$\Delta : \{-1, +1 \}^{p} \times \{-1, +1 \}^{p} \mapsto \mathbb{R}$,
we define a convex surrogate with an operator $\mathbf{B}$, 
\begin{equation}
\mathbf{B} \Delta : \{-1 , +1 \}^p \times \mathbb{R}^{p} \mapsto \mathbb{R} . \label{eq:ConvSurrogateOperatorSig}
\end{equation}
We may then minimize the empirical expectation of $\mathbf{B} \Delta(y,h(x))$ with respect to functions  $h : \mathcal{X} \mapsto \mathbb{R}^p$.  For well behaved function classes for $h$, minimization of the convex surrogate becomes tractable, provided that subgradient computation of $\mathbf{B} \Delta$ is efficiently solvable.

Any loss function $\Delta$ of this form may be interpreted as a set function where inclusion in a set is defined by a corresponding prediction being incorrect:
\begin{equation}\label{eq:delta_l}
\Delta(y,\tilde{y}) = l(\{i | y^i \neq \tilde{y}^i\}) 
\end{equation}
for some set function $l$.

In our analysis of convex surrogates for non-modular loss functions, we will employ several results for the Structured Output SVM~\cite{tsochantaridis2005large}, which assumes that a structured prediction is made by taking an inner product of a feature representation of inputs and outputs:
$\operatorname{sign}(h(x)) = \arg\max_y \left\langle w, \phi(x,y)\right\rangle$.
The slack rescaling variant of the Structured Output SVM is as follows:
\begin{align}
&\min_{w,\xi}  \frac{1}{2} \|w\|^2 + C \sum_{i=1}^{n} \xi_i,\quad \forall i,\forall \tilde{y}\in \mathcal{Y}: \label{eq:maxmargin}\\
&\langle w,\phi(x_i,y_i)\rangle
- \langle w,\phi(x_i,\tilde{y})\rangle \geq 1 - \frac{\xi_i}{\Delta(y_i,\tilde{y})} \label{eq:slackrescal}
\end{align}
In the sequel, we consider a feature function such that
$\langle w, \phi(x,y) \rangle = \sum_{j=1}^p \langle w^j , x^j \rangle y^j$. Each $w^j$ is then a vector of length $d$, and $w \in \mathbb{R}^{d\cdot p}$. Therefore $p$ individual prediction functions parametrized by $w^j$ are simultaneously optimized, although we may also consider cases in which we constrain $w^j=w^i\ \forall i,j$.  More generally, we may consider $h : \mathcal{X} \mapsto \mathbb{R}^p$, which may have non-linearities, e.g.\ deep neural networks.
 
\subsection{Mathematical preliminaries}
\begin{defin}
A set function $l$ maps from the powerset of some base set $V$ to the reals $l : \mathcal{P}(V) \mapsto \mathbb{R} $.
\end{defin}
\begin{defin}
A set function $l$ is non-negative if $l(A) - l(\emptyset) \geq 0, \ \forall A \subseteq V$.
\end{defin}
We denote the set of all such loss functions satisfying Equation~\eqref{eq:delta_l} $\mathcal{F}$.  Following standard conventions in submodular analysis, we assume that $l(\emptyset)=0$. In this paper we consider $l$ is non-negative, which we will denote $l \in \mathcal{F}_+$. 

\begin{defin}[Submodular function]\label{def:submodular}
  A set function $l : \mathcal{P}(V) \mapsto \mathbb{R}$ is submodular iff for all $B\subseteq A \subset V$ and $x \in V \setminus A$,
\begin{equation}
    l(B \cup \{x\}) - l(B) \geq l(A \cup \{x\}) - l(A)
\end{equation}
\end{defin}
A function is \emph{supermodular} iff its negative is submodular, and a function is modular (e.g. Hamming loss) iff it is both submodular and supermodular. 
We denote the set of all submodular functions as $\mathcal{S}$, and the set of all supermodular functions as $\mathcal{G}$.

\begin{defin}
A set function $l$ is symmetric if $l(A) = c(|A|)$ for some function $c : \mathbb{Z}^{*} \mapsto \mathbb{R} $.
\end{defin}

\begin{pro}
A symmetric set function $l$ is submodular iff $c$ is concave~\cite[Proposition~6.1]{MAL-039}.
\end{pro}

\begin{defin}[Increasing function]\label{def:increasing}
A set function $l:\mathcal{P}(V)\mapsto\mathbb{R}$ is increasing if and only if for all subsets $A\subset V$ and elements $x \in V \setminus A$, $l(A) \leq l(A \cup \{x\})$.
\end{defin}
We note that the set of increasing supermodular functions is identical to $\mathcal{G}_{+}$. We will propose a convex surrogate operator for a general non-negative loss function, based on the fact that set functions can always be expressed as the sum of a submodular function and a supermodular function:
\begin{pro}\label{pro:decompositionExists}
For all set functions $l$,  there always exists a decomposition into the sum of a submodular function $f\in\mathcal{S}$ and a supermodular function $g\in\mathcal{G}$:
\begin{equation}\label{eq:decomposition}
l = f +g 
\end{equation}
\end{pro}
A proof of this proposition is given in~\cite[Lemma~4]{narasimhan2005subsup}.
\begin{pro}\label{pro:decompositionIncreasingExists}
For an arbitrary decomposition $l=f+g$ where $g$ is not increasing, there exists a modular function $m_g$ s.t.
\begin{equation}
l=(f-m_g)+(g+m_g)
\end{equation}
with $\tilde{f}:=f-m_g\in\mathcal{S}$, and $\tilde{g}:=g+m_g\in\mathcal{G}_+$ is increasing.
\end{pro}
\begin{proof}
Any modular function can be written as
\begin{equation}\label{eq:modularIsLinear}
m_g(A) = \sum_{j \in A} w_j
\end{equation}
for some coefficient vector $w \in \mathbb{R}^{|V|}$.
For each $j \in V$, we may set
\begin{equation}\label{eq:ModularTransformIncreasing}
w_j = - \min_{A \subseteq V} g(A \cup \{j\}) - g(A) .
\end{equation}
The resulting modular function will ensure that $g+m_g$ is increasing following Definition~\ref{def:increasing}.
\end{proof}
This proof indicates that a decomposition $l = f+g$ is not-unique due to a modular factor.  We subsequently demonstrate that decompositions can vary by more than a modular factor:
\begin{pro}[Non-uniqueness of decomposition up to modular transformations.]\label{pro:nonUniqueNonModular}
For any set function, there exist multiple decompositions into submodular and supermodular components such that these components differ by more than a modular factor:
\begin{align}
&\exists f_1,f_2 \in \mathcal{S}, \ g_1,g_2 \in \mathcal{G}\nonumber\\
&\!\!\! \left( l = f_1+g_1 = f_2 + g_2 \right)\wedge\left(g_1 + m_{g_1} \neq g_2 + m_{g_2} \right)
\end{align}
where $\wedge$ denotes ``logical and,'' $m_{g_1}$ and $m_{g_2}$ are constructed as in Equations~\eqref{eq:modularIsLinear} and~\eqref{eq:ModularTransformIncreasing}.
\end{pro}
\begin{proof}
Let $m$ be a submodular function that is not modular.  For a given decomposition $l = f_1 + g_1$, we may construct $f_2 := f_1 + m$ and $g_2 := g_1 - m$.  As $m$ is not modular, there is no modular $m_{1}$ such that $g_1 - m_{1} = g_1 - m = g_2$.
\end{proof}

\section{A convex surrogate for general non-modular losses}\label{sec:convexsurrogate}

We will show in this section the unique decomposition for a general non-negative loss starting from any arbitrary submodular-supermodular decomposition, which allows us to define a convex surrogate operator based on such a canonical decomposition.

\subsection{A canonical decomposition}\label{sec:decomposition}

In this section, we define an operator $\mathbf{D}$ such that $g^*:=\mathbf{D}l\in\mathcal{G}_+$ is unique and $f^*:=l-\mathbf{D}l\in\mathcal{S}$ is then unique.
We have demonstrated in the previous section that we may consider there to be two sources of non-uniqueness in the decomposition $l = f+g$: a modular component and a non-modular component related to the curvature of $g$ (respectively $f$).  We define $\mathbf{D}$ such that these two sources of non-uniqueness are resolved using a canonical decomposition $l = f^* + g^*$.

\begin{defin}\label{def:decomposition}
We define an operator $\mathbf{D} : \mathcal{F} \mapsto \mathcal{G}_+$ as
\begin{equation}\label{eq:defD}
\mathbf{D} l = \arg\min_{g \in \mathcal{G}_+} \sum_{A \subseteq V} g(A), \quad \text{s.t. } l-g \in \mathcal{S} .
\end{equation}
\end{defin}

We note that minimizing the values of $g$ will simultaneously remove the non-uniqueness due both to the modular non-uniqueness described in Proposition~\ref{pro:decompositionIncreasingExists}, as well as the non-modular non-uniqueness described in Proposition~\ref{pro:nonUniqueNonModular}.  We formally prove this in Proposition~\ref{pro:uniquenessOfDl}.

\begin{pro}\label{pro:uniquenessOfDl}
$\mathbf{D} l$ is unique for all $l \in \mathcal{F}$ that have a finite base set $V$.
\end{pro}
\begin{proof}
We note that the $\arg\min$ in Equation~\eqref{eq:defD} is equivalent to a linear program: $g$ is uniquely determined by a vector in $\mathbb{R}^{2^{|V|}-1}$ the coefficients of which correspond to $g(A)$ for all $A \in \mathcal{P}(V) \setminus \emptyset$, and we wish to minimize the sum of the entries subject to a set of linear constraints enforcing supermodularity of $g$, non-negativity of $g$, and submodularity of $l-g$.

From~\cite[Theorem~2]{Mangasarian1979151}, an LP of the form
\begin{eqnarray}
\min_{x \in \mathbb{R}^{d}} & r^T x\\
\text{s.t. } & Cx \geq q
\end{eqnarray}
has a unique solution if there is no $y \in \mathbb{R}^d$ simultaneously satisfying
\begin{equation}\label{eq:LPuniqueConstraints}
C_J y \geq 0, \quad r^T y \leq 0, \quad y \neq 0
\end{equation}
where $J = \{i | C_i x^* = q_i \}$ is the active set of constraints at an optimum $x^*$.  We note that as $r$ is a vector of all ones (cf.\ Equation~\eqref{eq:defD}), $r^T y \leq 0$ constrains $y$ to lie in the non-positive orthant.  However, as the linear program is minimizing the sum of $x$ subject to lower bounds on each entry of $x$ (e.g.\ positivity constraints), we know that $C_J y \geq 0$ will bound $y$ to lie in the non-negative orthant.  This means, at most, these constraints overlap at $y=0$, but this is expressly forbidden by the last condition in Equation~\eqref{eq:LPuniqueConstraints}.
\end{proof}

Although Equation~\eqref{eq:defD} is a linear programming problem, we do not consider this definition to be constructive in general as the size of the problem is exponential in $|V|$ (see \cite{iyer2012algorithms}).  However, it may be possible to verify that a given decomposition satisfies this definition for some loss functions of interest.  Furthermore, for some classes of set functions, the LP has lower complexity, e.g.\ for symmetric set functions the resulting LP is of linear size, and loss functions that depend only on the number of false positives and false negatives (such as the S{\o}rensen-Dice loss discussed in Section~\ref{sec:DiceTheory}) result in a LP of quadratic size.

We finally note that from Equation~\eqref{eq:delta_l}, for every $\Delta(y, \cdot)$ we may consider its equivalence to a set function $l = g^* + f^*$, and denote the resulting decomposition of 
\begin{equation}\label{eq:LossDecompositionSupSub}
\Delta(y,\cdot) = \Delta_{\mathcal{G}}(y,\cdot) + \Delta_{\mathcal{S}}(y,\cdot)
\end{equation}
into its supermodular and submodular components, respectively.\footnote{Note that $\Delta_S$ and $\Delta_G$ are due to Eq.~\eqref{eq:delta_l} for $f^*$ and $g^*$ which explicitly depend on $\mathbf{D}$. For simplicity of notation, we will use $\Delta_{\mathcal{G}}$ instead of $\Delta_{\mathbf{D}\mathcal{G}}$}

\subsection{Definition of the convex surrogate}\label{sec:ConvexSurrogateDefinition}

Now that we have defined a unique decomposition $l = g^* + f^*$, we will use this decomposition to construct a surrogate $\mathbf{B} \Delta$ that is convex, piecewise linear, an extension of $\Delta$, and for which subgradient computation is polynomial time.  We construct a surrogate $\mathbf{B}$ by taking the sum of two convex surrogates applied to $\Delta_{\mathcal{G}}$ and $\Delta_{\mathcal{S}}$ independently.  These surrogates are slack-rescaling~\cite{tsochantaridis2005large} applied to $\Delta_{\mathcal{G}}$ and the Lov\'{a}sz hinge~\cite{yuICML2015} applied to $\Delta_{\mathcal{S}}$.
\begin{defin}[Slack-rescaling operator~\cite{yuICML2015}]\label{def:slackrescaling}
The slack-rescaling operator $\mathbf{S}$ is defined as:
\begin{equation}\label{eq:slackrescaling}
\mathbf{S}\Delta(y,h(x)):=\max_{\tilde{y} \in \mathcal{Y}}\Delta(y,\tilde{y})\left(1+ \langle h(x),\tilde{y} \rangle -\langle h(x),y \rangle \right).
\end{equation}
\end{defin}

The Lov\'{a}sz hinge of a submodular function builds on the Lov\'{a}sz extension~\cite{lovasz1983submodular}:
\begin{defin}[Lov\'{a}sz hinge~\cite{yuICML2015}]\label{def:lovaszHing}
The Lov\'{a}sz hinge, $\mathbf{L}$, is defined as the unique operator such that, for a submodular set function $l$ related to $\Delta$ as in Eq.~\eqref{eq:delta_l}:
\begin{align}
&\mathbf{L} \Delta(y,h(x)) := \nonumber\\
&\left(\max_{\pi}\sum_{j=1}^p s^{\pi_j} \left( l\left(\{\pi_1,\cdots,\pi_j\}\right)- l\left(\{\pi_1,\cdots,\pi_{j-1}\}\right)\right)\right)_+
\end{align}
where $(\cdot)_{+} = \max(\cdot,0)$, $\pi$ is a permutation,
\begin{equation}\label{eq:s_i_pi}
s^{\pi_j} = 1-h^{\pi_j}(x) y^{\pi_j},
\end{equation}
and $h^{\pi_j}(x)$ is the $\pi_j$th dimension of $h(x)$.
\end{defin}

\begin{defin}[General non-modular convex surrogate]
For an arbitrary non-negative loss function $\Delta$, we define 
\begin{equation}\label{eq:operatorB}
\mathbf{B}_{\mathbf{D}} \Delta := \mathbf{L} \Delta_{\mathcal{S}} + \mathbf{S} \Delta_{\mathcal{G}}
\end{equation}
where $\Delta_{\mathcal{S}}$ and $\Delta_{\mathcal{G}}$ are as in Equation~\eqref{eq:LossDecompositionSupSub}, and $\mathbf{D}$ is the decomposition of $l$ defined by Definition~\ref{def:decomposition}. 
\end{defin}

\mycorr{We use a cutting plane algorithm to solve the max-margin problem as shown in Algorithm~\ref{alg:cuttingplane}.}
\begin{algorithm}[!t]
\caption{Cutting plane algorithm }\label{alg:cuttingplane}
\begin{algorithmic}[1]
\STATE {Input: $(x_1,y_1),\cdots,(x_n,y_n),C,\epsilon$}
\STATE $S^i = \emptyset, \forall i = 1,\cdots,n$
\REPEAT 
\STATE for $i=1,\cdots,n$ do
\mycorr{
\STATE $ \hat{y}_L = \arg\max_{\tilde{y}} H_L(y_i) =  \arg\max_{\tilde{y}} \mathbf{L} \Delta_{\mathcal{S}} $
\STATE $ \hat{y}_S = \arg\max_{\tilde{y}} H_S(y_i) =  \arg\max_{\tilde{y}} \mathbf{S} \Delta_{\mathcal{G}} $
\STATE $ H(\hat{y}) =  H_L(\hat{y}_L) + H_S(\hat{y}_S)$ 
\STATE $\xi^i = \max\{0,H(y_i)\}$
}
\IF {$H(\hat{y})>\xi^i+\epsilon$ }
\STATE $S^i:=S^i\cup \{y_i\}$
\STATE $w\leftarrow$ optimize Equation~\eqref{eq:maxmargin} with constraints defined by $\cup_i S^i$
\ENDIF
\UNTIL{no $S^i$ has changed during an iteration}
\RETURN $(w,\xi)$
\end{algorithmic}  
\end{algorithm}

\subsection{Properties of $\mathbf{B}_{\mathbf{D}}$}\label{sec:ConvexSurrogateProperties}

In the remainder of this section, we show that $\mathbf{B}_{\mathbf{D}}$ has many desirable properties.  Specifically, we show that $\mathbf{B}_{\mathbf{D}}$ is closer to the convex closure of the loss function than slack rescaling and that it generalizes the Lov\'{a}sz hinge (Theorems~\ref{theo:decompositionsupermodular} and~\ref{theo:decompositionsubmodular}).  Furthermore, we formally show that $\mathbf{B}_{\mathbf{D}} \Delta$ is convex (Theorem~\ref{theo:convex}), an extenstion of $\Delta$ for a general class of loss functions (Theorem~\ref{theo:extension}), and polynomial time computable (Theorem~\ref{theo:polytime}).

\begin{lemma}\label{lemma:decompositionsuper1}
If $l\in\mathcal{G}$, then $f^*:=l-\mathbf{D}l \in \mathcal{S}\cap\mathcal{G}$ i.e.\ modular.
\end{lemma}
\begin{proof}
First we set $l=g_m+f_m$ where $f_m$ is modular and $f_m(\{j\}) = l(\{j\})$. Then any subset $S\subseteq V$
\begin{align}
& f_m(S) = \sum_{j\in S} l(\{j\}), \  g_m(S) = l(S) - \sum_{j\in S} l(\{j\}) \nonumber \\
& \sum_{S\subseteq V} g_m(S) = \sum_{S\subseteq V} \left( l(S) - \sum_{j\in S} l(\{j\}) \right) .
\label{eq:decompositionsuper1}
\end{align}
The sum in Equation~\eqref{eq:decompositionsuper1} is precisely the sum that should be minimized in Equation~\eqref{eq:defD}.  We now show that this sum cannot be minimized further while allowing $f_m$ to be non-modular.
If there exists any $g$ s.t.\ $f:=l-g$ is submodular but not modular, by definition there exists at least one subset $S_s\subseteq V$ and one $j\in S_s$ such that
\begin{equation}
 f(S_s\setminus\{j\})+f(\{j\}) > f(S_s)+f(\emptyset) .
\end{equation}
Then by subtracting each time one element from the subset $S_s$, we have
\begin{align}
f(S_s) & < f(S_s\setminus\{j\})+f(\{j\}) \nonumber\\
& \leq f(S_s\setminus(\{j\}\cup\{k\}))+f(\{k\})+f(\{j\})\nonumber \\
& \leq\cdots\leq\sum_{j\in S_s} f(\{j\})),  \qquad \forall k\in S_s\setminus(\{j\} 
\end{align}
which implies
\begin{equation}
g(S_s) > l(S_s) - \sum_{j\in S_s} l(\{j\})\label{eq:decompositionsuper2}
\end{equation}
By taking the sum of the inequalities as in Equation~\ref{eq:decompositionsuper2} for all subsets $S$, we have that 
\begin{equation*}
\sum_{S\subseteq V} g(S) > \sum_{S\subseteq V} \left( l(S) - \sum_{j\in S} l(\{j\})\right) = \sum_{S\subseteq V} g_m(S) 
\end{equation*}
which means $\sum_{S\subseteq V} g(S) > \sum_{S\subseteq V} g_m(S) $ for any $g$. By Definition~\ref{def:decomposition}, $g^* = g_m =  \mathbf{D}l$, thus $f^*:=l-\mathbf{D}l = f_m$ is modular.
\end{proof}

\begin{lemma}\label{lemma:decompositionsuper2}
For a loss function $\Delta$ such that $\Delta_{\mathcal{S}}$ is increasing, we have
\begin{equation}
\mathbf{S} \Delta_{\mathcal{G}} =
\mathbf{S}(\Delta-\Delta_\mathcal{S})= \mathbf{S}\Delta - \mathbf{S} \Delta_\mathcal{S}.
\end{equation}
\end{lemma}
\begin{proof}
By Definition~\ref{def:slackrescaling}, for every single cutting plane determined by some $\tilde{y}$, we have
\begin{align}
&\mathbf{S}\left(\Delta(y,\tilde{y})-\Delta_\mathcal{S}(y,\tilde{y})\right)\nonumber \\
& = (\Delta(y,\tilde{y})-\Delta_\mathcal{S}(y,\tilde{y})) \left(1+ \langle h(x),\tilde{y} \rangle -\langle h(x),y \rangle \right)\nonumber \\
& = \Delta(y,\tilde{y}) \left(1+ \langle h(x),\tilde{y} \rangle -\langle h(x),y \rangle \right) \nonumber \\
&\quad - \Delta_\mathcal{S}(y,\tilde{y}) \left(1+ \langle h(x),\tilde{y} \rangle -\langle h(x),y \rangle \right) \nonumber \\
& = \mathbf{S}\Delta(y,\tilde{y})- \mathbf{S}\Delta_\mathcal{S}(y,\tilde{y}) .
\end{align}
As this property holds for all cutting planes, it also holds for the supporting hyperplanes that define the convex surrogate and $\mathbf{S}(\Delta-\Delta_\mathcal{S})= \mathbf{S}\Delta - \mathbf{S} \Delta_\mathcal{S}$.
\end{proof}

\begin{defin}\label{def:extension}
A convex surrogate function $\mathbf{B} \Delta(y,\cdot)$ is an extension when 
\begin{equation} \label{eq:extension}
\mathbf{B} \Delta(y,\cdot) = \Delta(y,\cdot)
\end{equation} on the vertices of the 0-1 unit cube under the mapping to $\mathbb{R}^p$: 
$i=\{1,\dots,p\},\  [u]^i = 1-h^{\pi_i}(x) y^{\pi_i}$
\end{defin}

\begin{theorem}\label{theo:decompositionsupermodular}
If $l\in\mathcal{G}_+$, then $\mathbf{B}_{\mathbf{D}} \Delta \geq \mathbf{S} \Delta$ over the unit cube given in Definition~\ref{def:extension}, and therefore $\mathbf{B}_{\mathbf{D}}$ is closer to the convex closure of $\Delta$ than $\mathbf{S}$.
\end{theorem}
\begin{proof}
By the definition of the Lov\'{a}sz hinge $\mathbf{L}$~\cite{yuICML2015}, we know that for any modular function $\Delta_{\mathcal{S}}$ we have $\mathbf{L}\Delta_\mathcal{S} \geq \mathbf{S}\Delta_\mathcal{S}$ over the unit cube. As a result of Lemma~\ref{lemma:decompositionsuper1} and Lemma~\ref{lemma:decompositionsuper2},we have
\begin{align}
\mathbf{B}_{\mathbf{D}} \Delta & = \mathbf{S}\Delta_\mathcal{G} +  \mathbf{L}\Delta_\mathcal{S} \nonumber 
 = \mathbf{S}(\Delta-\Delta_\mathcal{S})  +  \mathbf{L}\Delta_\mathcal{S} \nonumber \\
& \geq \mathbf{S}\Delta- \mathbf{S}\Delta_\mathcal{S}   +  \mathbf{S}\Delta_\mathcal{S} \nonumber  = \mathbf{S}\Delta .
\end{align}
\end{proof}

\begin{theorem}\label{theo:decompositionsubmodular}
If $l\in\mathcal{S}$, then $\mathbf{B}_{\mathbf{D}}\Delta=\mathbf{L}\Delta$
\end{theorem}
\begin{proof}
For $l\in\mathcal{S}$, we construct $g^*=\mathbf{0}$, and $f^*=l$ is submodular. By Definition~\ref{def:decomposition}, $g^*(V)$ is minimum, so $g^*=\mathbf{D}l$. Then $\mathbf{B}_{\mathbf{D}}\Delta=\mathbf{L}\Delta + \mathbf{S}\mathbf{0}=\mathbf{L}\Delta$
\end{proof}

\begin{theorem}\label{theo:convex}
$\mathbf{B}_{\mathbf{D}}\Delta$ is convex for arbitrary $\Delta$.
\end{theorem}
\begin{proof}
By Definition~\ref{def:decomposition}, $\mathbf{B}_{\mathbf{D}}\Delta$ is the sum of the two convex surrogates, which is a convex surrogate.
\end{proof}

\begin{theorem}\label{theo:extension}
$\mathbf{B}_{\mathbf{D}}\Delta$ is an extension of $\Delta$ iff $\Delta_{\mathcal{S}}$ is non-negative.
\end{theorem}
\begin{proof}
From \cite[Proposition~1]{yuICML2015}, $\mathbf{S}\Delta$ is an extension for any supermodular increasing $\Delta$; $\mathbf{L}\Delta$ is an extension iff $\Delta$ is submodular and non-negative as in this case, $\mathbf{L}$ coincides with the Lov\'{a}sz extenstion~\cite{lovasz1983submodular}. By construction from Definition~\ref{def:decomposition} we have $\Delta_{\mathcal{G}}$ and $\Delta_{\mathcal{S}}$ for $g\in\mathcal{G}_+$ and $f\in\mathcal{S}$, respectively. Thus Equation~\ref{eq:extension} holds for both $\mathbf{S}\Delta_{\mathcal{G}}$ and $\mathbf{L}\Delta_{\mathcal{S}}$ if $\Delta_{\mathcal{S}}$ is non-negative. Then $\mathbf{B}_{\mathbf{D}}\Delta$ taking the sum of the two extensions, Equation~\ref{eq:extension} also holds for every vertex of the unit cube as $\Delta = \Delta_{\mathcal{G}}+\Delta_{\mathcal{S}}$, which means $\mathbf{B}_\mathbf{D}$ is also an extension of $\Delta$ .
\end{proof}

\begin{theorem}\label{theo:polytime}
The subgradient computation of $\mathbf{B}_\mathbf{D} \Delta$ is polynomial time given polynomial time oracle access to $f^*$ and $g^*$.
\end{theorem}
\begin{proof}
Given $f^*$ and $g^*$ we know that the subgradient computation of $\mathbf{L}\Delta_{\mathcal{S}}$ and $\mathbf{S}\Delta_{\mathcal{G}}$ are each polynomial time. Thus taking the sum of the two is also polynomial time.
\end{proof}

\section{S{\o}rensen-Dice loss}\label{sec:DiceTheory}
The S{\o}rensen-Dice criterion~\cite{dice45,Sorensen-1948-BK} is a popular criterion for evaluating diverse prediction problems such as image segmentation~\cite{sabuncu2010generative} and language processing~\cite{rychly2008lexicographer}. In this section, we introduce the S{\o}rensen-Dice loss based on the S{\o}rensen-Dice coefficient.  We prove that the S{\o}rensen-Dice loss is neither supermodular nor submodular, and we will show in the experimental results section that our novel convex surrogate can yield improved performance on this measure.

\begin{defin}[S{\o}rensen-Dice Loss] Denote $y \subseteq V$ is a set of positive labels, e.g.\ foreground pixels, 
the S{\o}rensen-Dice loss on given a groundtruth $y$ and a predicted output $\tilde{y}$ is defined as
\begin{equation}
\Delta_D(y,\tilde{y}) = 1-  \frac{2|y\cap \tilde{y}|}{|y|+|\tilde{y}|}. \label{eq:diceloss}
\end{equation}
\end{defin}

\begin{pro}\label{pro:dicemodularity}
$\Delta_D(y,\tilde{y})$ is neither submodular nor supermoduler under the isomorphism $(y^*,\tilde{y}) \rightarrow A := \{i | y_i^* \neq \tilde{y}_i\}$, $\Delta_{J}(y^*,\tilde{y}) \cong l(A)$.  
\end{pro}

We will use the diminishing returns definition of submodularity in Definition~\ref{def:submodular} to first prove the following lemma:
\begin{lemma}\label{lem:falsenegative}
  $\Delta_D$ restricted to false negatives is neither submodular nor supermoduler.
\end{lemma}
\begin{proof}
With the notation $m := |y^*| > 0$, $p := | \tilde{y} \setminus y^* |$, and $n := | y^* \setminus \tilde{y} |$, we have that
\begin{equation}
\Delta_D(y,\tilde{y}) = 1 - \frac{2m-2n}{2m-n+p} = \frac{n+p}{2m-n+p}
\end{equation}
For a given groundtruth $y$ i.e.\ $m$, we have if $B\subseteq A$, then $n_B \leq n_A$, and $p_B\leq p_A$.

Considering $i$ is an extra false negative, we calculate the marginal gain on $A$ and $B$ respectively:
\begin{align}
&\Delta_D(A\cup\{i\})-\Delta_D(A) \nonumber\\
& = \frac{n_A+1+p_A}{2m-n_A-1+p_A} -\frac{n_A+p_A}{2m-n_A+p_A} \\
& = \frac{2m+2p_A}{(2m-n_A+p_A-1)(2m-n_A+p_A)}\label{eq:L_D_A_gain}\\
&\Delta_D(B\cup\{i\})-\Delta_D(B) \nonumber\\
& = \frac{2m+2p_B}{(2m-n_B+p_B-1)(2m-n_B+p_B)}.\label{eq:L_D_B_gain}
\end{align}
Numerically, we have following counter examples which prove that $\Delta_D$ restricted to false negatives is neither submodular nor supermodular. We set $m=10,\ n_A=[1:8],\ n_B=n_A-1 \leq n_A,\ 
 p_A=8,\ p_B=5\leq p_A$, and we plot the values of Equation~\eqref{eq:L_D_A_gain} and Equation~\ref{eq:L_D_B_gain} as a function of $n_A$. We can see from Figure~\ref{fig:counterexample} that there exists a cross point between these two plots, which indicates that submodularity (Definition~\ref{def:submodular}) does not hold for $\Delta_{D}$ or its negative.

\end{proof}

Lemma~\ref{lem:falsenegative} implies Proposition~\ref{pro:dicemodularity} as the restriction of a submodular function is itself submodular.

\begin{figure}[]\centering
\begin{minipage}{0.54\columnwidth}
\includegraphics[width=0.95\linewidth]{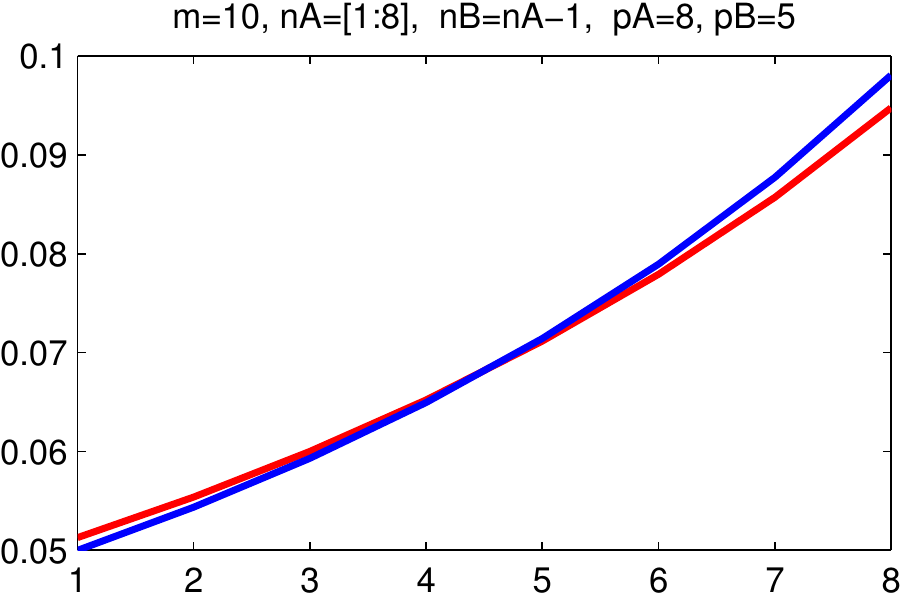}
\caption{Plots of Eq.~\eqref{eq:L_D_A_gain} (red) and Eq.~\ref{eq:L_D_B_gain} (blue) as a function of $n_A$. As these two plots cross, neither function bounds the other
.}\label{fig:counterexample}
\end{minipage}\hfill
\begin{minipage}{0.445\columnwidth}
\includegraphics[width=0.95\linewidth]{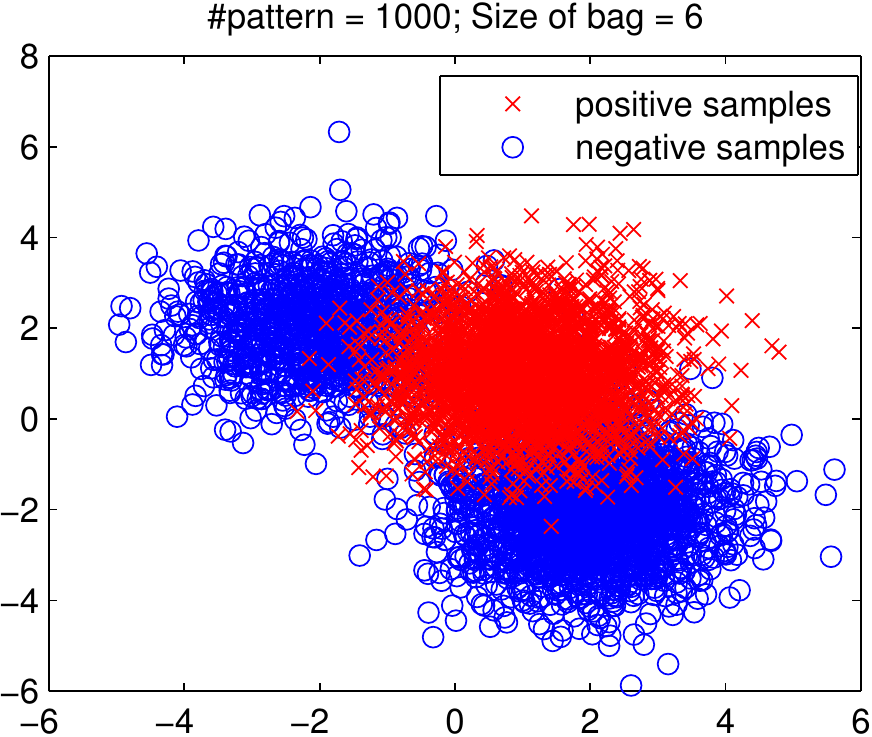}
\caption{\label{fig:toy}The  data for the synthetic problem. The negative samples are drawn from a mixture of Gaussians.}	
\end{minipage}\hfill
\end{figure}

\section{Experimental Results}\label{sec:result}

We demonstrate the correctness and feasibility of the proposed convex surrogate on experiments using Dice loss, as well as on a face classification problem from video sequences with a family of non-modular losses.
\subsection{Dice loss}
We test the proposed surrogate on a binary set prediction problem. Two classes of 2-dimensional data are generated by different Gaussian mixtures as shown in Fig~\ref{fig:toy}.
We use the $\mathbf{B}_{\mathbf{D}}$ during training time with the non-modular loss $\Delta_D$ to construct a convex surrogate. We compare it to slack rescaling $\mathbf{S}$ with an approximate optimization procedure based on greedy maximization. We additionally train an SVM (denoted 0-1 in the results table) for comparison. During test time, we evaluate with $\Delta_D$ and with Hamming loss to calculate the empirical error values as shown in Table~\ref{tab:losstoy}. 

We can see from the result that 
training $\Delta_D$ with $\mathbf{B}_{\mathbf{D}}$ yields the best result while using $\Delta_D$  during test time. $\mathbf{B}_{\mathbf{D}}$ performs better than $\mathbf{S}$ in both cases due to the failure of the approximate maximization procedure necessary to maintain computational feasibility~\cite{krause2012survey}. 

\begin{table*}[]\centering
\begin{minipage}{0.36\linewidth}
\resizebox{0.99\linewidth}{!}{\centering
\begin{tabular}{ c|c|c}
\hline
 & \multicolumn{2}{|c}{Test} \\
\cline{2-3}
$p=6$ & $\Delta_D$ & 0-1  \\
\hline
\multirow{3}{*}{ }
  $\mathbf{B}_{\mathbf{D}}$ &	$\mathbf{0.1121 \pm 0.0040}$  &  $0.6027\pm 0.0125$  \\
  0-1                       &	$0.1497\pm 0.0046$             & $0.5370\pm 0.0114$  \\
  $\mathbf{S}$              &	$0.3183\pm 0.0148$             & $0.7313\pm 0.0209$  \\
\hline
\end{tabular}}
\caption{For the synthetic data experiment, the cross comparison of average loss values (with standard error) using different surrogate operations during training, and different evaluation functions during test time. $\Delta_D$ is the Dice loss as in Eq.~\eqref{eq:diceloss}. 
}\label{tab:losstoy}
\end{minipage}\hfill
\begin{minipage}{0.63\linewidth}\centering
\resizebox{0.99\linewidth}{!}{
\begin{tabular}{ c|c|c|c|c}
\hline
  & \multicolumn{4}{|c}{ loss functions } \\
\cline{2-5}
 & $\Delta_1$ & $\Delta_2$ &  $\Delta_3$ ($\Delta_{\mathcal{S}}$ negative) & $\Delta_4$  ($\Delta_{\mathcal{S}}$ negative) \\
\hline
  $\mathbf{B}_\mathbf{D}$ &	$\mathbf{0.194 \pm 0.006}$ &	$\mathbf{0.238 \pm 0.008}$ &	$0.148 \pm 0.005$  &	$0.108 \pm 0.004$    \\
  0-1          &	$0.228 \pm 0.007$  &	$0.284 \pm 0.004$  &	$0.144 \pm 0.004$   &	$0.107 \pm 0.003$          \\
  $\mathbf{S}$ &	$0.398 \pm 0.015$  & $0.243 \pm 0.005$    &	$0.143 \pm 0.006$   &	$0.106 \pm 0.003$    \\
\hline
\end{tabular}}
\caption{\label{tab:error}For the face classification task, the cross comparison of average loss values (with standard error) using different surrogate operator and losses as in Equation~\eqref{eq:loss1} to Equation~\eqref{eq:loss4} during training, respectively. For the cases that the submodular component is non-negative, i.e.\ using $\Delta_1$ and $\Delta_2$, the lowest empirical error is achieved when using $\mathbf{B}_\mathbf{D}$.}
\end{minipage}\hfill
\end{table*}

\subsection{Face classification in video sequences}
\begin{figure*}
\centering
\subfigure[$\Delta_1$]{\includegraphics[width=0.2\linewidth]{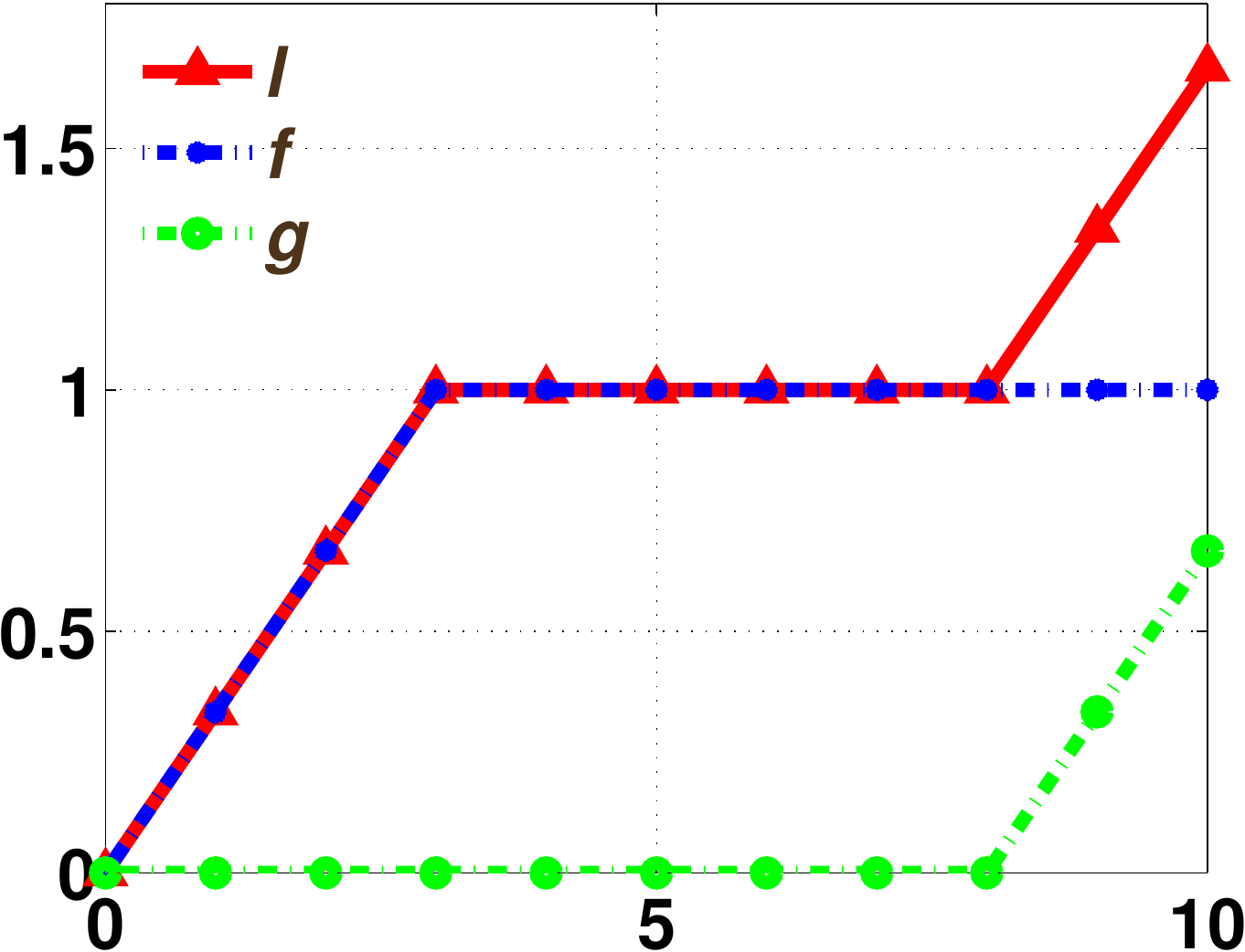}\label{fig:loss1}}
\subfigure[$\Delta_2$]{\includegraphics[width=0.2\linewidth]{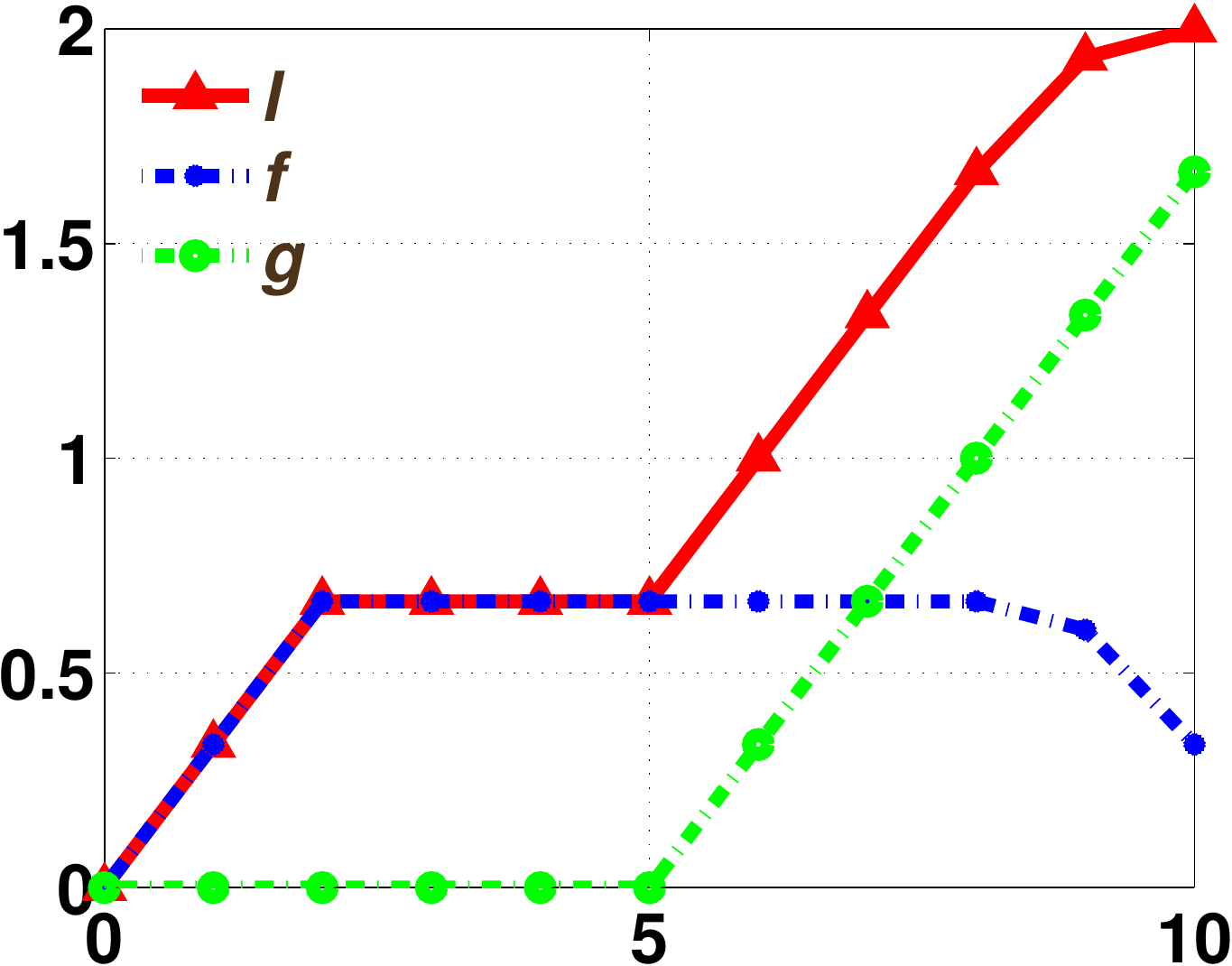}\label{fig:loss2}}
\subfigure[$\Delta_3$]{\includegraphics[width=0.2\linewidth]{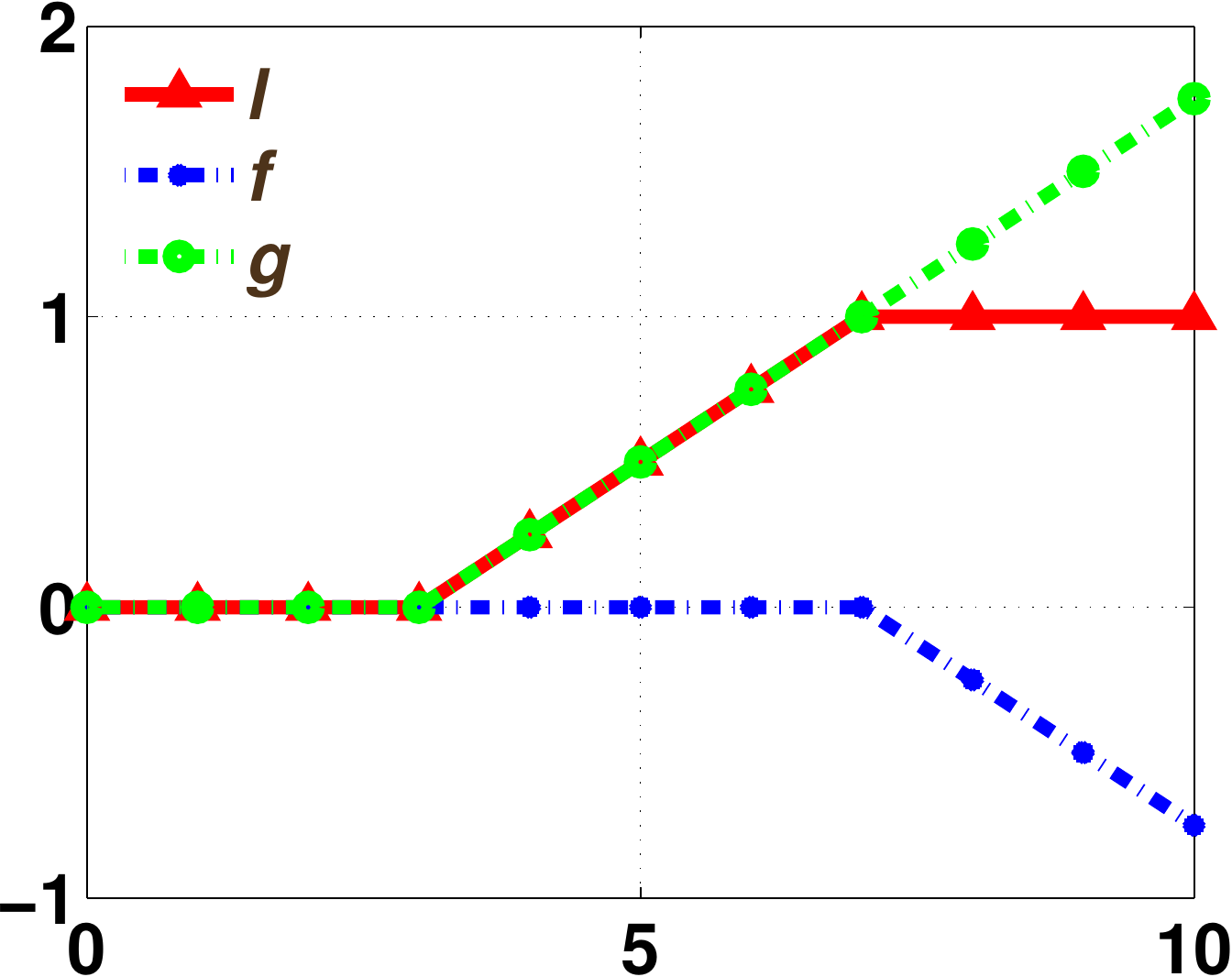}\label{fig:loss3}}
\subfigure[$\Delta_4$]{\includegraphics[width=0.21\linewidth]{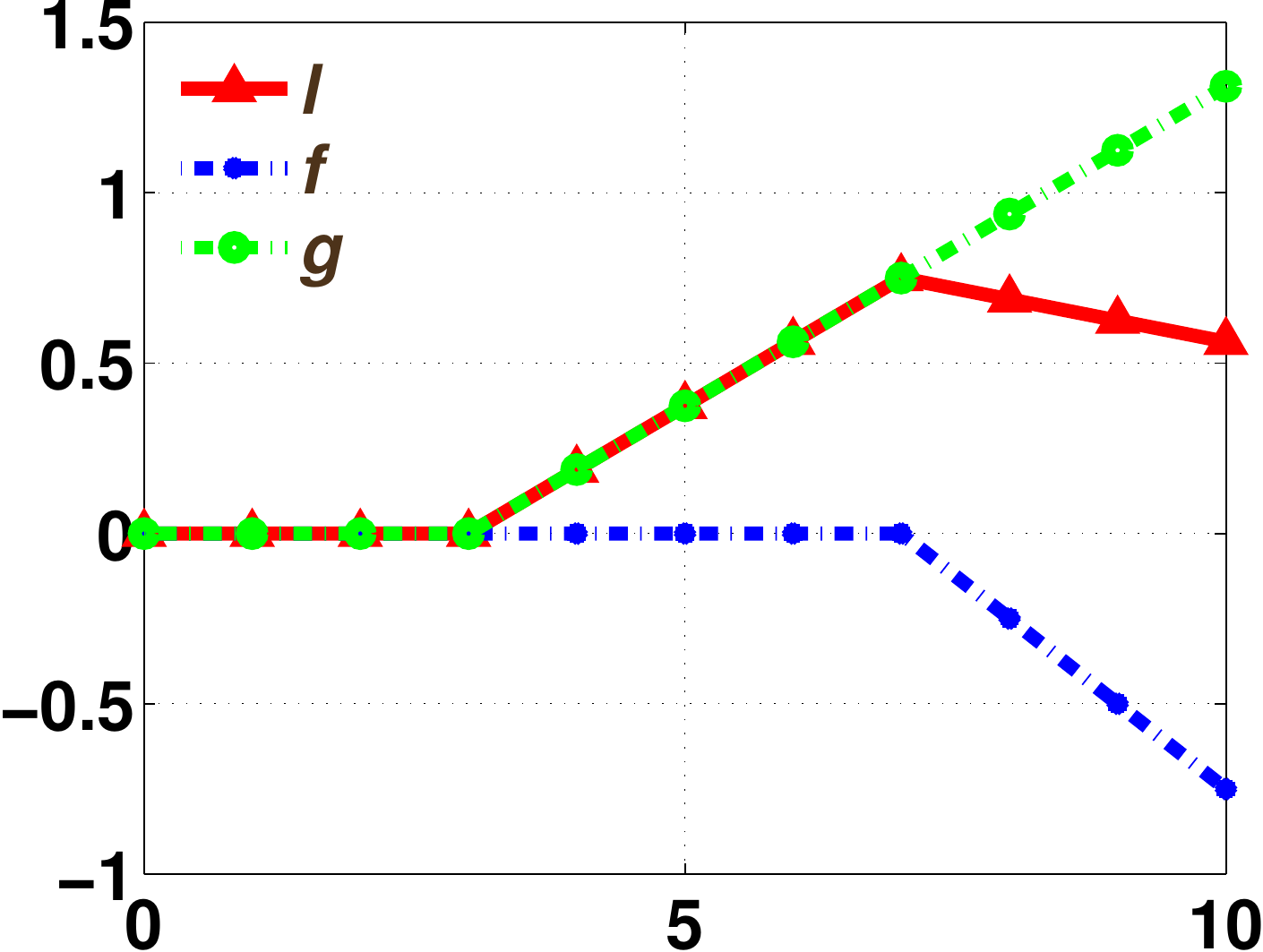}\label{fig:loss4}}
\caption{\label{fig:lossfunction} The plot of the four loss functions used in our experiments as in Equations~\eqref{eq:loss1} to~\eqref{eq:loss4}. The $x$ axis is the number of mispredictions for each track (we show here the loss functions corresponding to track length equal to $10$ as an example), and the $y$ axis is the value of loss function.
The original losses are drawn in {\color{red}red}; the supermodular components are drawn in {\color{green}green}, and the submodular components in {\color{blue}blue}. } 
\end{figure*}

\begin{figure}  \centering
\subfigure[ ]{\includegraphics[width=0.32\linewidth]{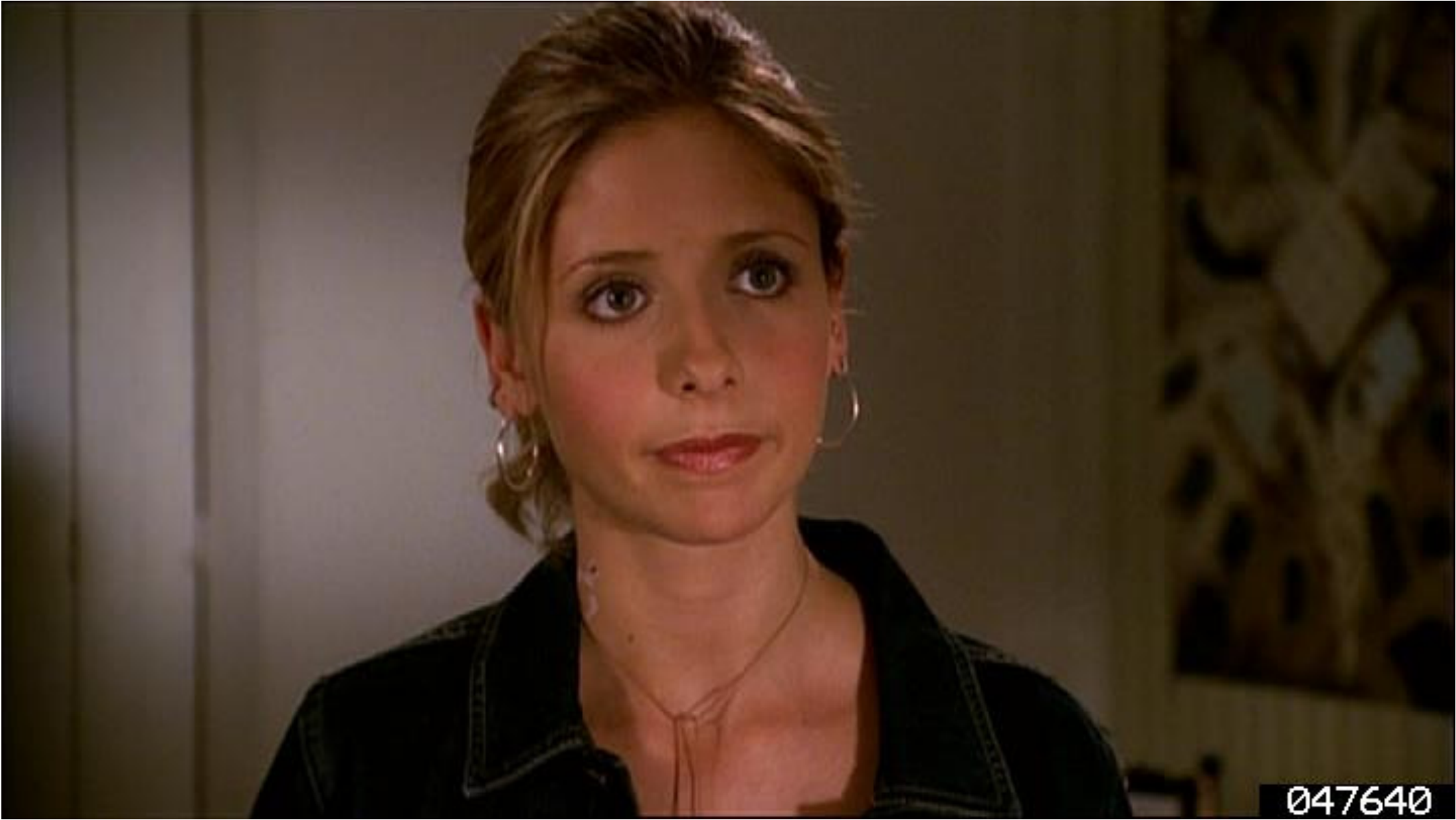}\label{fig:face2}}
\subfigure[ ]{\includegraphics[width=0.32\linewidth]{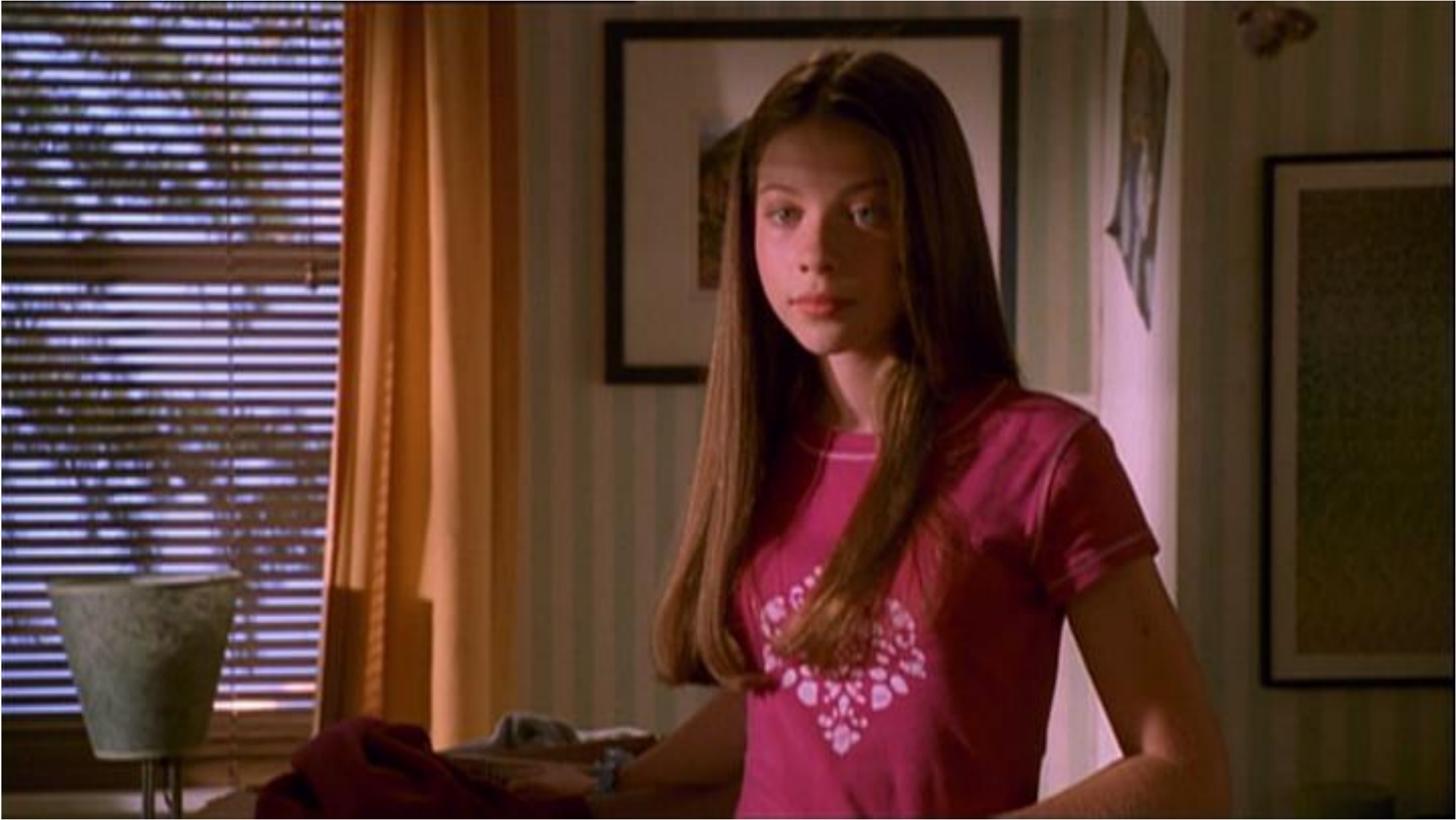}\label{fig:face3}}
\subfigure[ ]{\includegraphics[width=0.18\linewidth]{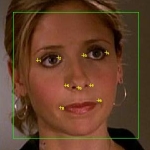}\label{fig:faceWithBox}}
\caption{\label{fig:face}Examples of the face track images. Fig.~\ref{fig:face2} shows the ``Buffy'' role thus a positive-labelled image and Fig.~\ref{fig:face3} shows a negative-labelled image.  An automated pipeline described in~\cite{Everingham06a,Everingham09,Sivic09} was used for feature extraction (Fig.~\ref{fig:faceWithBox}).}	
\end{figure}

We also evaluate the proposed convex surrogate operator on a real-world face track dataset~\cite{Everingham06a,Everingham09,Sivic09}. The frames of the dataset are from the TV series ``Buffy the Vampire Slayer''. This dataset contains 1437 tracks and 27504 frames in total. 

We focus on a binary classification task to recognize the leading role: ``Buffy'' is positive-labelled, ``not Buffy'' is negative-labelled. Example images are shown in Fig.~\ref{fig:face}. Each track is represented as a bag of frames, for which the size of the tracks varies from 1 frame to more than 100 frames, and each image is represented as a Fisher Vector Face descriptor of dimension $1937$. 

We have used different non-supermodular and non-submodular loss functions in our experiments as shown in Equations~\ref{eq:loss1} to~\ref{eq:loss4}: 
\begin{align}
&\Delta_1(y,\tilde{y}) =  \min(|\mathbf{I}| , |y|/3 , |\mathbf{I}| - |y|/3 )\label{eq:loss1}\\
&\Delta_2(y,\tilde{y}) =  \min(|\mathbf{I}| , |y|/4 , |\mathbf{I}| - |y|/4, \alpha )\label{eq:loss2}\\
&\Delta_3(y,\tilde{y}) =  \min( \max(0,|\mathbf{I}| -|y|/3 ), |y|/3 )\label{eq:loss3}\\
&\Delta_4(y,\tilde{y}) =  \min(\max(0,|\mathbf{I}| -|y|/3 ),\alpha   )\label{eq:loss4}
\end{align}
$\mathbf{I}=\{i|y^i\neq\tilde{y}^i\}$ gives the set of incorrect prediction elements; $\alpha $ is a parameter that allows us to define the value of $l(V)$. Due to the fact that the size of the tracks varies widely, we further normalize the loss function with respect to the track size. We use $\alpha =2$ for $\Delta_2$ and $\alpha =0.5$ for $\Delta_4$ in the experiments.

As we can see explicitly in Fig.~\ref{fig:lossfunction}, no $\Delta$ is supermodular or submodular. $\Delta_1$, $\Delta_2$ and $\Delta_3$ are increasing loss functions, while $\Delta_4$ is non-increasing. For $\Delta_1$ and  $\Delta_2$, we notice that the values of the set functions on a single element are non-zero i.e.\ $l_1(\{j\})>0,\ l_2(\{j\})>0,\ \forall j\in V$; while for the loss $\Delta_3$ and $\Delta_1$ these values are zero i.e.\ $l_3(\{j\})=l_4(\{j\})=0,\ \forall j\in V$. 


Fig.~\ref{fig:lossfunction} shows the corresponding decomposition of each loss into the supermodular and submodular components as specified in Definition~\ref{def:decomposition}.
We denote each loss function as $l_k = f_k + g_k$, for $k = \{1,2,3,4\}$. 


By construction, all supermodular $g_k$, for $k = \{1,2,3,4\}$, are non-negative increasing. For the submodular component, $f_1$ is non-negative increasing, $f_2$ is non-negative and non-increasing, while $f_3$ and $f_4$ are both non-positive decreasing. 

We compare different convex surrogates during training for these non-modular functions. And we additionally train on the Hamming loss (labelled 0-1) as a comparison. As training non-supermodular loss with slack rescaling is NP-hard, we have employed the simple application of the greedy approach as in~\cite{krause2012survey}.

\begin{figure*}
\centering
\subfigure[The gap using $\Delta_1$]{\includegraphics[width=0.23\linewidth]{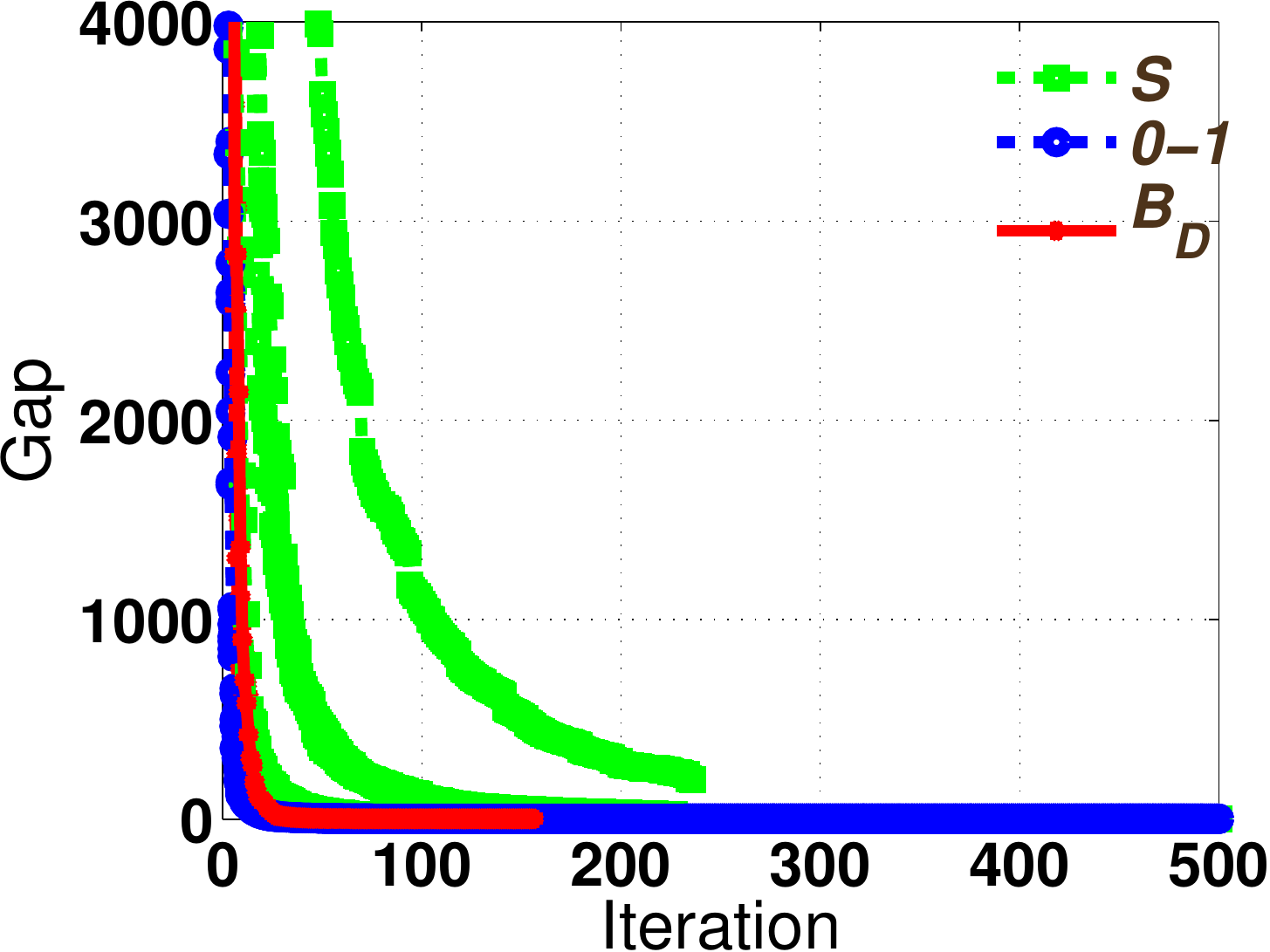}\label{fig:gap1}}
\subfigure[The gap using $\Delta_2$]{\includegraphics[width=0.23\linewidth]{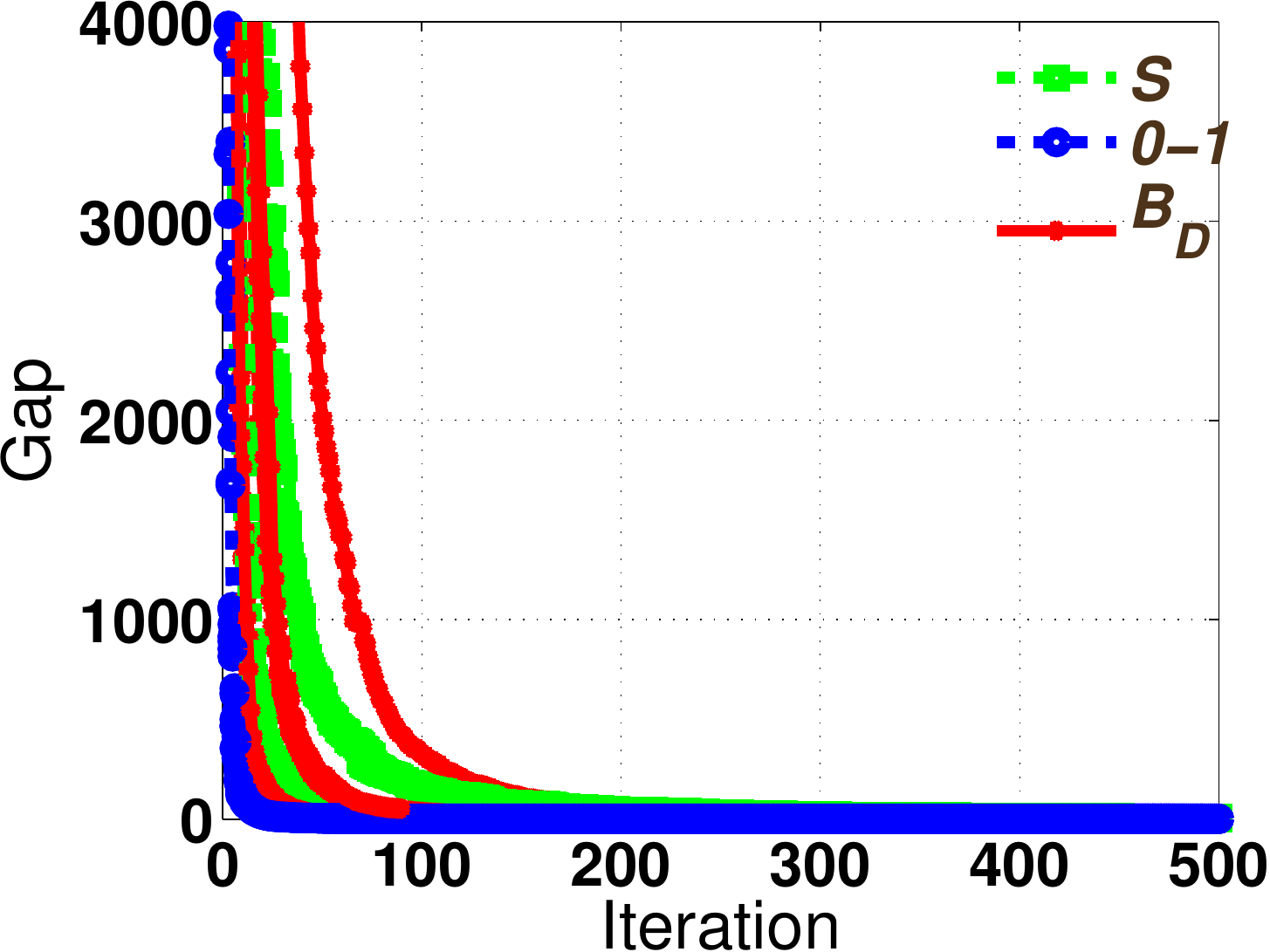}\label{fig:gap2}}
\subfigure[The gap using $\Delta_3$]{\includegraphics[width=0.23\linewidth]{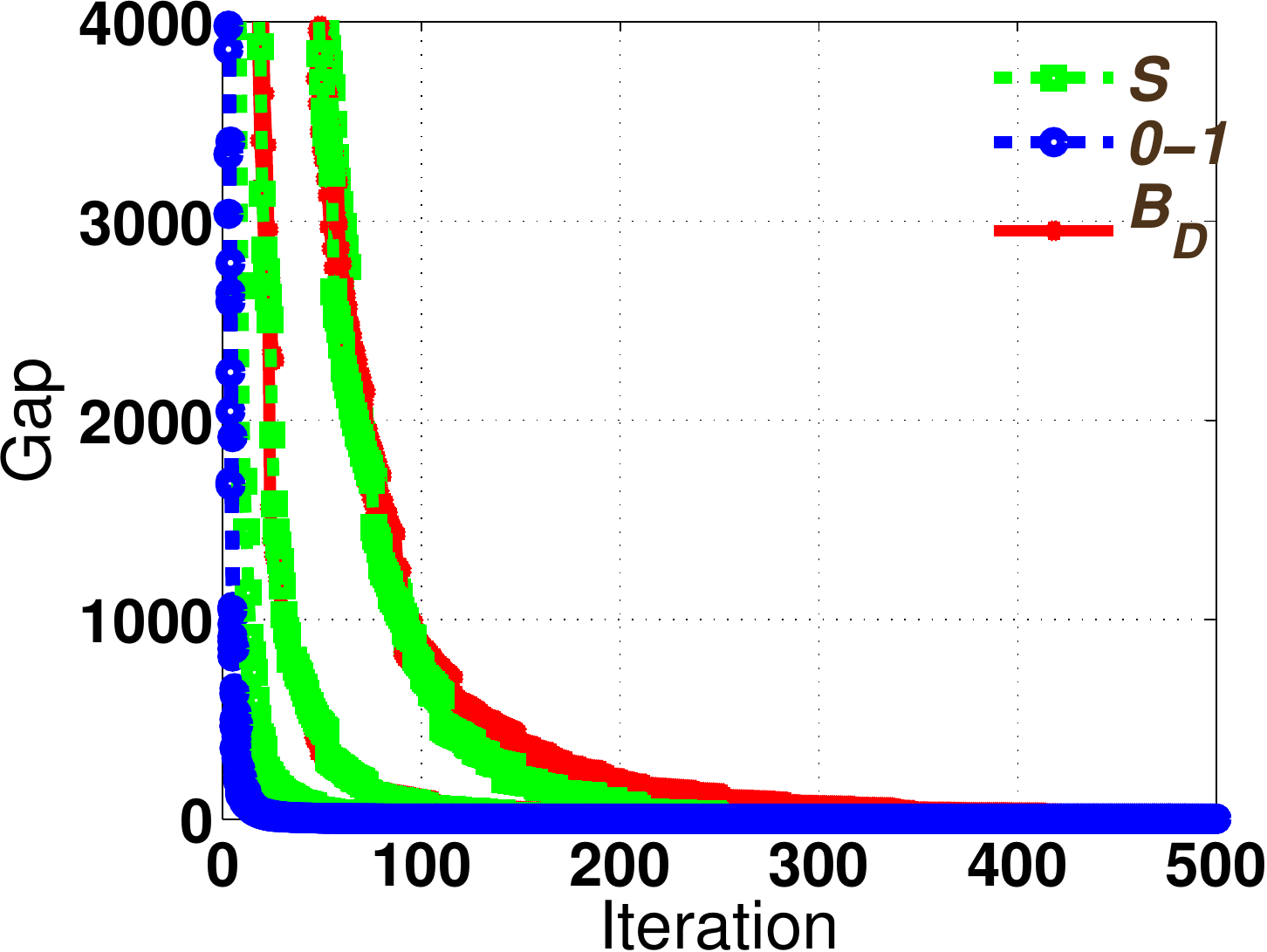}\label{fig:gap3}}
\subfigure[The gap using $\Delta_4$]{\includegraphics[width=0.23\linewidth]{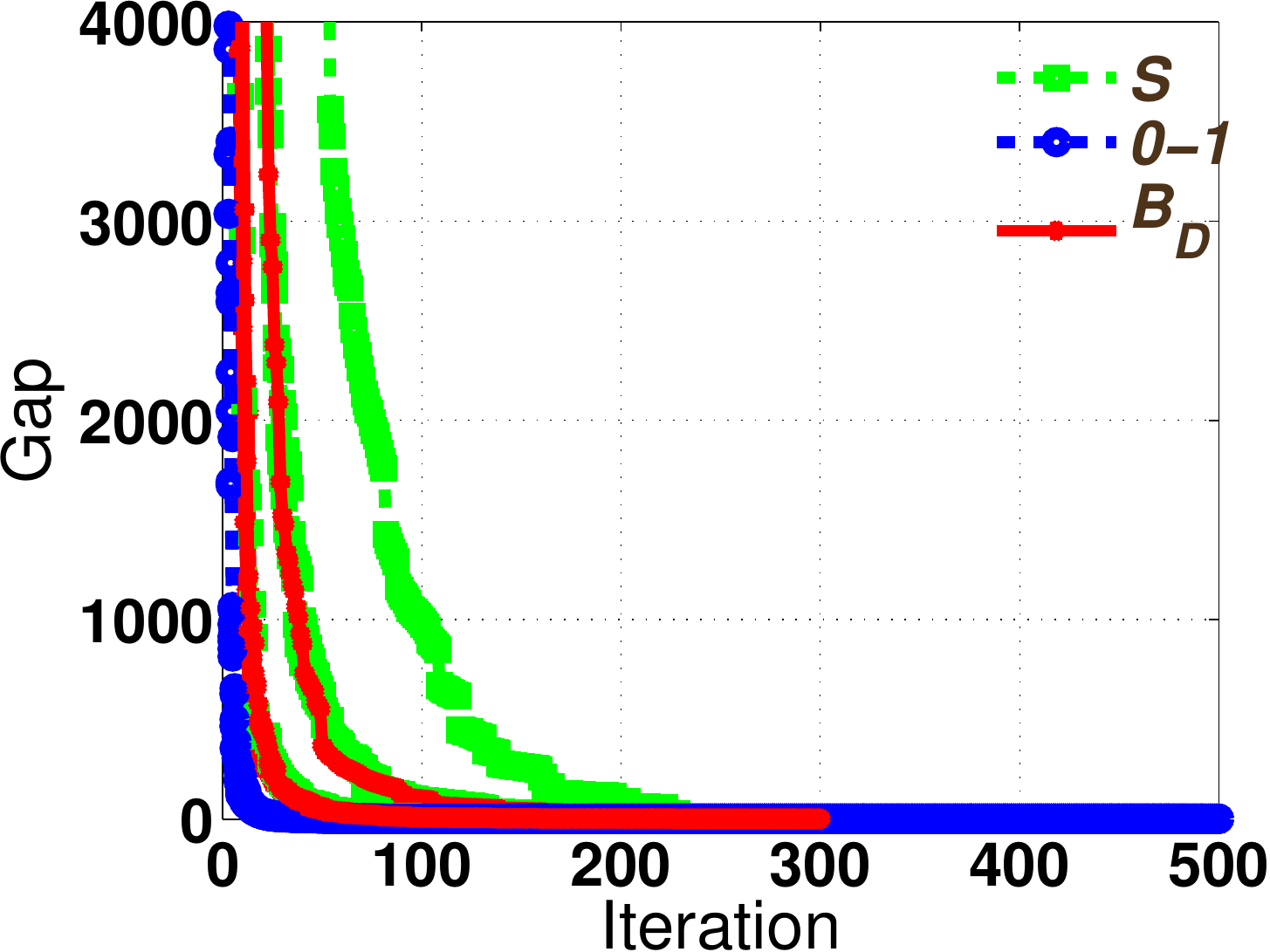}\label{fig:gap4}}
\caption{\label{fig:gap} The primal-dual gap as a function of the number of cutting-plane iterations using different convex surrogates for the four non-modular functions in Equations~\eqref{eq:loss1} to~\ref{eq:loss4}. The primal-dual gap from $\mathbf{B}_\mathbf{D}$ is drawn in {\color{red}red}; the gap from $\mathbf{S}$ is drawn in {\color{green}green}, and gap from Hamming loss (labelled 0-1, and equivalent to a SVM) in {\color{blue}blue}. Our convex surrogate operator $\mathbf{B}_\mathbf{D}$ can achieve a comparable convergence rate to an SVM, demonstrating that optimization is very fast in practice and the method scales well to large datasets.} 
\end{figure*}

10-fold-cross-validation has been carried out and we obtain an average performance and standard error as shown in Table~\ref{tab:error}.

From Table~\ref{tab:error} we can see that when the submodular component of the decomposition is non-negative, i.e.\ in the case of using $\Delta_1$ and $\Delta_2$, the lowest empirical error is achieved by using our convex surrogate operator $\mathbf{B}_\mathbf{D}$.

Fig.~\ref{fig:gap} shows the primal-dual gap as a function of the cutting plane iterations for each experiment using different loss functions and different convex surrogate operators. We can see that in all cases, the convergence of $\mathbf{B}_\mathbf{D}$ is at a rate comparable to an SVM, supporting the wide applicability and scalability of the convex surrogate. 
\begin{table}[]\centering
\resizebox{0.9\linewidth}{!}{
\begin{tabular}{c|c|c|c}
\hline
		 &  $p=10$   &  $p=50$   & $p=100$  \\ 
\hline 
 $\mathbf{B}_\mathbf{D}$		& $0.002\pm 0.000$ & $0.018\pm 0.003$ & $0.060\pm 0.008$ \\
\hline
$\mathbf{S}$					& $0.002\pm 0.000$ & $0.016\pm 0.002$ & $0.057\pm 0.002$\\
\hline
\end{tabular}}
\caption{\label{tab:time} The comparison of the computation time (s) for one loss augmented inference.}
\end{table}
We have also compared the expected time of one loss augmented inference. Table~\ref{tab:time} shows the comparison using $\Delta_1$ with $\mathbf{B}_\mathbf{D}$ and $\mathbf{S}$.  As the cost per iteration is comparable to slack-rescaling, and the number of iterations to convergence is also comparable, there is consequently no computational disadvantage to using the proposed framework, while the statistical gains are significant.  

\section{Discussion and Conclusions}

The experiments have demonstrated that the proposed convex surrogate is efficient, scalable, and reduces test time error for a range of loss functions, including the S{\o}rensen-Dice loss, which is a popular evaluation metric in many problem domains.  We see that slack rescaling with greedy inference can lead to poor performance for non-supermodular losses.  This is especially apparent for the results of training with $\Delta_1$, in which the test-time loss was approximately double that of the proposed method.  Similarly, ignoring the loss function and simply training with 0-1 loss can lead to comparatively poor performance, e.g.\ $\Delta_1$ and $\Delta_2$.  This clearly demonstrates the strengths of the proposed method for non-modular loss functions for which a decomposition with a non-negative submodular component is possible ($\Delta_1$ and $\Delta_2$, but not $\Delta_3$ or $\Delta_4$).  The characterization and study of this family of loss functions is a promising avenue for future research, with implications likely to extend beyond empirical risk minimization with non-modular losses as considered in this paper.  The primal-dual convergence results empirically demonstrate that the loss function is feasible to apply in practice, even on a dataset consisting of tens of thousands of video frames.  The convex surrogate is directly amenable to other optimization techniques, such as stochastic gradient descent~\cite{bottou-bousquet-2008}, or Frank-Wolfe approaches~\cite{ICML2013_lacoste-julien13}, as well as alternate function classes including neural networks.

In this work, we have introduced a novel convex surrogate for general non-modular loss functions. 
We have defined a decomposition for an arbitrary loss function into a supermodular non-negative function and a submodular function.  We have proved both the existence and the uniqueness of this decomposition. 
Based on this decomposition, we have proposed a novel convex surrogate operator taking the sum of two convex surrogates that separately bound the supermodular component and the submodular component using slack-rescaling and the Lov\'{a}sz hinge, respectively.  
We have demonstrated that our new operator is a tighter approximation to the convex closure of the loss function than slack rescaling, that it generalizes the Lov\'{a}sz hinge, and is convex, piecewise linear, an extension of the loss function, and for which subgradient computation is polynomial time.  Open-source code of $\ell_2$ regularized risk minimization with this operator is available for download from \projecturl.

\subsection*{Acknowledgements}
This work is partially funded by Internal Funds KU Leuven, ERC Grant 259112,
and FP7-MC-CIG 334380.  The first author is supported by a fellowship from
the China Scholarship Council.

\bibliographystyle{abbrv}
\bibliography{nips2015,biblio}

\end{document}